\documentclass[11pt]{article}

\usepackage{amsfonts,psfrag,epsfig}
\usepackage{amsfonts,epsfig}
\usepackage{color}
\usepackage{amsmath,amssymb,hyperref}
\usepackage[amsmath,thmmarks]{ntheorem}

\usepackage[margin = 1.1in]{geometry}

\hypersetup{colorlinks=true}

\newtheorem{theorem}{Theorem}
\newtheorem{lemma}[theorem]{Lemma}

\newtheorem{definition}{Definition}

{ \theoremstyle{nonumberplain}\theoremsymbol{\ensuremath{\Box}}
\theorembodyfont{\normalfont}
\newtheorem{proof}{Proof.}

\theoremsymbol{\ensuremath{\Box}}
\theoremheaderfont{\normalfont\itshape}
\theorembodyfont{\normalfont} \theoremstyle{empty}

}

\newcommand{\beq}{\begin{eqnarray}}
\newcommand{\eeq}{\end{eqnarray}}
\newcommand{\beqn}{\begin{equation}}
\newcommand{\eeqn}{\end{equation}}
\newcommand{\R}{\mathbb{R}}

\newcommand{\mb}{\mathbf}

\newcommand{\bx}{\mb{x}}

\newcommand{\by}{\mb{y}}

\newcommand{\bz}{\mb{z}}

\newcommand{\y}{\boldsymbol{y}}

\newcommand{\edges}{E}
\newcommand{\n}{\mathcal{N}}

\title{The Complexity of Approximating a Bethe Equilibrium\footnote{The preliminary version of this paper was presented at 
International Conference on Artificial Intelligence and Statistics (AISTATS) 2012.}}
\author{ \begin{tabular}{cc}
Jinwoo Shin\footnote{IBM T. J. Watson Research, Yorktown Heights, NY 10598. Email:
mijirim@gmail.com.}
\end{tabular}}
\begin{document}

\maketitle

\begin{abstract}

This paper resolves a common complexity issue in the Bethe
approximation of statistical physics and the Belief
Propagation (BP) algorithm of artificial intelligence. The Bethe
approximation and the BP algorithm are heuristic methods for estimating the partition
function and marginal probabilities in graphical models, respectively.
The computational complexity of the Bethe approximation is decided by
the number of operations required to solve a set of non-linear
equations, the so-called Bethe equation. %, and it is known that
%On the other hand, the BP
%algorithm is a popular heuristic method for estimating marginal
%distributions in a graphical model. 
Although the BP algorithm was inspired and
developed independently, %from different directions, 
Yedidia, Freeman and Weiss (2004) showed that 
the BP algorithm solves the
Bethe equation if it converges (however, it often does not). This
naturally motivates the following question to understand
limitations and empirical successes of the Bethe and BP methods: is the Bethe equation
computationally easy to solve?

We present a message-passing algorithm solving the Bethe equation in
a polynomial number of operations for general binary graphical
models of $n$ variables where the maximum degree in the underlying
graph is $O(\log n)$. Our algorithm can be used as an alternative to BP fixing its
convergence issue and is the first fully polynomial-time approximation
scheme for the BP fixed-point computation in such a large class of
graphical models, while the approximate fixed-point computation is
known to be (PPAD-)hard in general. We believe that our technique is
of broader interest to understand the computational complexity of
the cavity method in statistical physics.

\end{abstract}
%\newpage
%\setcounter{page}{1}

\section{Introduction}\label{sec:intro}

In recent years, graphical models (also known as Markov random
fields) %and message-passing algorithms
defined on graphs have been studied as powerful formalisms modeling
inference problems in numerous areas including computer vision,
speech recognition, error-correcting codes, protein structure,
networking, statistical physics, game theory and combinatorial
optimization. Two central problems, commonly addressed in these
applications involving graphical models, are computing the marginal
distribution and the so-called partition function. It is well-known
that inference problems are computationally hard in general
\cite{Chand08}. Due to such a theoretical barrier, efforts have been
made to develop heuristic methods.

The sum-product algorithm, also known as Belief Propagation (BP), 
and its variants (e.g., Survey Propagation) are
such heuristics, driven by certain experimental thoughts, for
computing the marginal distribution, where BP was first
proposed by Gallager \cite{Gallager} for error correcting codes and Pearl \cite{Pearl} for artificial intelligence.
Their appeal lies in the ease
of implementation as well as optimality in tree-structured graphical
models (models which contain no cycles). BP (and message-passing
algorithms in general) can be thought as an updating rule on a set
of messages:
$$m^{t+1}~:=~f\left(m^t\right),$$
where $m^t$ is the multi-dimensional vector of messages at the
$t$-th iteration, and $f$ describes the updating rule (or BP
operator). Two major hurdles to understand such a message-passing
algorithm are its convergence (i.e., does $m^t$ converge to
$m^*$?) and correctness (i.e., is $m^*$ good enough?). It is known
that the BP iterative procedure always has a fixed-point $m^*$ due
to the Brouwer's fixed-point theorem. However, BP can oscillate far
from a fixed-point in models with cycles, and only 
sufficient convergence conditions \cite{Weiss, Jordan, Heskes,
Willsky} have been established in the last decade. More importantly,
BP can have multiple fixed-points, and even when the fixed-point is unique, it
may not be the correct answer. Significant efforts \cite{Heskes,
Wainwright, YWY04} were made to understand BP fixed-points, while
the precise approximation qualities and the rigorous understandings
on their limitations still remain a mystery. Regardless of those
theoretical understandings, the BP algorithm performs empirically
well in many applications \cite{Freeman, Murphy}. For
example, the highly successful turbo decoding algorithm \cite{Berrou} 
can be interpreted as BP \cite{MMC98} and decisions guided by BP are
also known to work well to solve satisfiability problems
\cite{Ricci}. %However, the BP algorithm often does not converge
%(i.e., it provides no answer). Its performance still remains mystery
%for models with cycles.

The Bethe approximation \cite{Bethe,YWY04} and its variants (e.g.,
Kikuchi approximation \cite{DG72}), originally developed in
statistical physics of lattice models, are currently used as
powerful approximation schemes for computing the (logarithm of the)
partition function in many applications. The Bethe approximation
suggests to use the following quantity as an approximation for the
logarithm of the partition function:
$$\qquad\qquad\qquad\qquad\qquad\qquad\qquad F(\by^*)
\qquad\qquad\qquad\mbox{where}\quad\nabla F(\by^*)=0.$$ 
Here, $F$, $\nabla
F(\by^*)=0$ and $\by^*$ are called the (minus) Bethe free energy
function, Bethe equation and Bethe equilibrium, respectively. The
statistical physics prediction suggests its asymptotic correctness
in random sparse graphical models, and several rigorous evidences in
particular models are known \cite{bandyopadhyay2006counting,
dembo2010ising, ShinBP}. Efforts have also been made to estimate and
characterize its error \cite{LoopC, ruozzi2012bethe, SWW}. However, the error still
remains uncontrollable for models with many cycles.

Yedidia, Freeman and Weiss \cite{YWY04} established a somewhat
surprising connection between the BP algorithm and the Bethe
approximation: if BP converges, it solves the Bethe equation.
Equivalently, the BP fixed-point equation $f(m^*)=m^*$ is in essence
equivalent to the Bethe equation $\nabla F(\by^*)=0$. This naturally
leads to the following common computational question for both: is the
BP fixed-point computation computationally easy? Formally
speaking,
\begin{itemize}
\item[$\mathcal{Q}$.] Given $\varepsilon>0$, is it possible to design a deterministic iterative algorithm
finding $m^*$ satisfying
\begin{equation*}
(1-\varepsilon)\, f(m^*)~\leq~{m^*}~\leq~(1+\varepsilon)\, f(m^*),\footnote{Inequalities are with respect to all coordinates
	of vectors $m^*$, $(1-\varepsilon) f(m^*)$ and $(1+\varepsilon) f(m^*)$. Note that it is impossible to
	compute the exact solution $m^*$ with $f(m^*)=m^*$ since it is
	irrational in general.}%\label{eq:fixcomp}
\end{equation*}
in a polynomial number of bitwise operations with respect to
$1/\varepsilon$ and the dimension of the vector $m^*$? 
\end{itemize}
Such an algorithm can be used as an alternative to BP with provably
fast convergence rate (i.e., fixing the convergence issue of BP) and
eliminates a need for the convergence analysis of BP. Even though it
may not converge to the precise answer, rough estimations on marginal probabilities
are sometimes enough to solve hard computational problems (see \cite{Ricci}, for example). Further, it justifies that
the Bethe approximation is a polynomial-time scheme since $m^*$ satisfying the
above inequality provides $\by^*$ with $\|\nabla F(\by^*)\|\leq
\varepsilon$. Efforts to design such algorithms were made \cite{Teh,
Yuille02}, but no rigorous analysis
on their convergence rates is known. %(their simulation results were promising though). %We also note that
The authors in \cite{ShinBP} provide an algorithm with provable
polynomial convergence rate, but the work is for a very specific
graphical model (i.e., the uniform distribution on independent sets
of sparse graphs). It is far from being clear whether such a
poly-convergence algorithm exists for more general graphical models.
This is primarily because 
the Bethe function is usually neither convex nor concave (cf., \cite{vontobel2010bethe}) and
computing a local minimum (or a fixed-point)
approximately are known to be believably (PPAD- or PLS-)hard in general 
\cite{daskalakis2011continuous}.\footnote{PPAD and PLS are computational classes capturing
the hardnesses of finding fixed-points and locally optimal solutions, respectively. They have gained
much attention in the field of algorithmic game theory in the last decade
under the connection with the computational complexity of Nash equilibria.}

\subsection{Our Contribution}\label{sec:con}
The main result of this paper is the following answer $\mathcal{A}$
(see Theorem \ref{thm:runningtime} in Section \ref{sec:two}) for
the question $\mathcal{Q}$ for the BP operator $f$ and general
sparse binary graphical models of which potential functions are bounded above and below by some positive constants. 
To state it formally, we let $n$ be
the number of nodes and $\Delta$ be the maximum degree in the
underlying graph, respectively.
\begin{itemize}
\item[$\mathcal{A}$.] Given $\varepsilon>0$, there exists a deterministic iterative algorithm finding
$m^*$ satisfying
\begin{equation*}
(1-\varepsilon)\, f(m^*)~\leq~{m^*}~\leq~(1+\varepsilon)\,
f(m^*)
\end{equation*} in $2^{O(\Delta)}n^2 \varepsilon^{-4} \log^3 (n
\varepsilon^{-1})$ iterations.
\end{itemize}
In this paper, we call the message $m^*$ satisfying the above
inequality an $\varepsilon$-approximate BP fixed-point. In what
follows, we explain the algorithm in details.

The known equivalence \cite{YWY04} between the BP fixed-point
equation and the Bethe equation implies that the question
$\mathcal{Q}$ is equivalent to the following.
\begin{itemize}
\item[$\mathcal{Q}^{\prime}$.] Given $\varepsilon^{\prime}>0$, is it possible to design a deterministic iterative algorithm
finding $\by^*$ satisfying
\begin{equation*}
\|\nabla F (\by^*)\|~\leq ~\varepsilon^{\prime},
\end{equation*}
in a polynomial number of bitwise operations with respect to
$1/\varepsilon^{\prime}$ and the dimension of the domain $D$ of the Bethe
free energy function $F$?
\end{itemize}
However, we remind the reader that it is still far from being
obvious whether it is computationally `easy' to find such a near-stationary 
point or an approximate local minimum (or maximum).
Natural attempts are gradient-descent algorithms to find a local
minimum or maximum of $F$: iteratively update $\by(t)$ as
$$\by(t+1)~:=~\by(t) + \alpha\,\nabla F( \by(t) ),$$
where $\alpha\in \mathbb{R}$ is the (appropriately chosen)
step-size. The main issue here is that the gradient-descent
algorithm may not find a near-stationary point if $\by(t)$ hits the
boundary of $D$ in one of its iterations (and a projection is
required). Hence, the main strategy in \cite{ShinBP} to avoid the
hitting issue lies in (a) understanding the behavior of the gradient
$\nabla F$ close to the boundary of $D$ and (b) designing an
appropriate small step-size in the gradient-descent algorithm based
on the understanding (a).

Now we give an overview of our technical contributions.
The main challenge to apply the strategy to general binary
graphical models (beyond the specific model in \cite{ShinBP}) is on
(a). The main observation used in \cite{ShinBP} is that  
the domain $D$ can be reduced to $\left[0,\frac12\right]^n$ 
in the uniform independent-set model.
In general, the dimension of $D$ is 
is much larger than
$n$ (i.e., the number of nodes) %($m$ is the number of edges in the underlying graph)
since the parameter $\by$ of the Bethe function $F$ represents not only node marginal probabilities but also
edge (i.e., pairwise) ones. However, in independent-set models, pairwise marginal probabilities are decided by node marginal probabilities,
which allows to reduce the dimension of $D$ to $n$.
 % by appropriately redefining the Bethe function $F$. 
%in
%\cite{ShinBP} since the Bethe free energy function $F$ is determined
%by node marginal probabilities in the uniform independent-set model.
The proof strategy in \cite{ShinBP} crucially relies on $D=\left[0,\frac12\right]^n$
%sensitive to the assumption $D=\left[0,\frac12\right]^n$ 
and immediately fails even for a non-uniform
independent-set model whose domain $D$ is reduced to $[0,1]^n\not\subset \left[0,\frac12\right]^n$.
Furthermore, in general binary graphical models, such a dimension reduction in $D$ is impossible and % the dimension of $D$ 
%the dimension of $D$ can be much larger than $n$
%in general graphical
%models.  %This is
%because the Bethe free energy should consider pairwise (or edge)
%marginal probabilities as well. 
it is not hard to check that any similar approaches with \cite{ShinBP} fail without it. 
To overcome such a technical issue, we first
observe that at stationary points of $F$, pairwise marginal
probabilities should satisfy certain quadratic equations in terms of
node marginal probabilities in binary graphical models. Hence, 
one can express the Bethe free energy again in terms of node marginal
probabilities (i.e., a dimension reduction in $D$
is possible) for the purpose of obtaining a (near-)stationary point of $F$.
Now we study this `modified' Bethe
expression $F^*$ to avoid the hitting issue, which we end up with an
appropriate small step-size in the gradient-descent algorithm.
Moreover, we eliminate a need to decide such a small step-size
explicitly in the algorithm, by designing a time-varying
projection scheme.

We later realize that the `modified' Bethe expression $F^*$ was
already proposed by Teh and Welling \cite{Teh}, where they suggested
gradient algorithms to minimize $F^*$ using sigmoid functions. The
main difference in our work is that we study the behavior of
the gradient $\nabla F^*$ close to the boundary of its domain and
guarantee that the gradient-descent algorithm does not hit the
boundary of the underlying domain $D$, i.e., we do not use sigmoid functions. The success of our rigorous
convergence rate analysis, which was missing in the work of Teh and
Welling (2001), primarily relies on this difference. It is also
crucial to extend the algorithm design to non-binary graphical
models as we describe in Section \ref{sec:three}.

%This is possible since our algorithm design was motivated from
%avoiding the hitting issue.

%The main difference in our work is that our gradient algorithm
%minimizes it using a projection scheme, while they suggested to use
%sigmoid functions.
%This is primarily because our algorithm design is motivated
%differently from avoiding the hitting issue.
%This is crucial for our rigorous convergence rate analysis and
%extension to non-binary graphical models, which are beyond the work
%\cite{Teh}.

One can observe that our gradient-descent algorithm is implementable
as a `BP-like' iterative, message-passing algorithm: each node
maintains a message at each iteration and passes it to its
neighbors. If potential functions in binary graphical models 
are bounded above and below by some positive constants (i.e., their values are $\Theta(1)$),
we prove it terminates in $2^{O(\Delta)}n^2
\varepsilon^{-4} \log^3 (n \varepsilon^{-1})$ iterations 
until it finds an $\varepsilon$-approximate BP
fixed-point (see Theorem \ref{thm:runningtime} in Section \ref{sec:two}).
In a complexity point of view, the only remaining
issue is that each node may require to maintain irrational messages
(of infinitely long bits). We further show that a polynomial number
(with respect to $1/\varepsilon$, $n$ and $2^{\Delta}$) of bits to
approximate each message suffices, and hence the algorithm consists
of only a polynomial number of bitwise operations in total. Namely,
it is a fully polynomial-time approximation scheme (FPTAS) to
compute an approximate BP fixed-point for sparse binary graphical models where $\Delta=O(\log n)$. 
Finally, we note that our `quadratic' running-time guarantee (i.e., $2^{O(\Delta)}n^2
\varepsilon^{-4} \log^3 (n \varepsilon^{-1})$) is merely a theoretical bound, and far from being tight.
In our experimental results reported in Section \ref{sec:simul}, 
we observe that our algorithm sometimes converges faster than the standard BP algorithm.

%We
%emphasize that the existence of such a strongly polynomial time
%algorithm is not clear even in popular convex optimization schemes,
%such as LP (linear programming) and SDP (semidefinite programming).
%We believe that our rigorous complexity analysis is an important
%step to bring the cavity method of statistical physics into computer
%science.

 \iffalse One of important implications of our result
is about the complexity of the Bethe approximation \cite{YWY04}
which suggests to use the Bethe free energy at a (near-)stationary
point as an approximation to the partition function of the graphical
model. Chertkov and Chernyak \cite{LoopC} propose an explicit
formula (called loop calculus) to compute the error-term of the
Bethe approximation. In the Bethe approximation scheme and the loop
calculus, the first necessary step is on understanding the
(computational) complexity to find the stationary point, which is
addressed in this paper.\fi

\subsection{Organization}

In Section \ref{sec:one}, we provide backgrounds for graphical
models, Belief Propagation and Bethe approximation. In Section
\ref{sec:two}, we describe our algorithm and its time complexity for
binary graphical models. In Section \ref{sec:three}, we discuss, at a high level, how
to extend the result to non-binary graphical models.
From our discussion in Section \ref{sec:three}, one can observe that
it is not hard to obtain the similar convergence rate result for such graphical models as
well. But, we omit the further details in this paper.
Experimental results are reported in Section \ref{sec:simul}.

\section{Graphical Models}\label{sec:one}

We first introduce a class of joint distributions defined with
respect to (undirected) graphs, which are called %\emph{pairwise graphical models}
pairwise {\em Markov random fields} (MRFs) \cite{Lauritzen1996}.\footnote{We note that
for any (directed or undirected) graphical model, there exists an equivalent pairwise MRF.}
Specifically, let $G = (V, E)$ be an undirected graph with vertex set V where $|V|=n$, and edge set $E
\subseteq {V\choose2}$ denoting a set of unordered pairs of
vertices. The vertices of $G$ label a collection of random variables
$\bx =
\{x_v\,|\, v \in V\}$. %For any subset of vertices $A \subseteq V$, $x_A
%= \{x_v| v \in A\}$.
Our primary focus in this paper is on binary random variables, i.e.,
$x_v \in \{0,1\}$ for all $v\in V$.

Now consider the following joint distribution on $\{0,1\}^n$ that
factors according to $G$:
%\begin{equation}
$$p(\bx) = \frac{1}{Z} \prod_{v \in V} \psi_v(x_v) \prod_{(u,v) \in
\edges} \psi_{u,v}(x_u,x_v)\qquad\mbox{for}~~\bx\in\{0,1\}^n.$$
%\label{eq:gm}
%\end{equation}
Here, $\psi_{u,v}$ for $(u,v)\in E$ and $\psi_v$ for $v\in V$ are non-negative functions on
$\{0,1\}^2$ and $\{0,1\}$, respectively. These local functions are
called \emph{potential} functions or \emph{compatibility} functions.
The normalizing factor $Z$ is called the \emph{partition} function:
\begin{equation}
Z = \sum_{\bx \in \{0,1\}^{n}} \prod_{v \in V} \psi_v(x_v)
\prod_{(u,v) \in \edges} \psi_{u,v}(x_u,x_v). \label{eq:part}
\end{equation}

%\textcolor{red} Hence, we consider
%In this paper, we the following pairwise potential potential
%functions.
%\begin{eqnarray}
%\psi_v(x_v) ~=~ 1, ~~\forall\, v \in V \quad & & \quad
%\psi_{u,v}(x_u,x_v) ~=~ 1 - x_ux_v, ~~\forall\, (u,v) \in
%\edges.\label{eq:ispotential}
%\end{eqnarray}
%Graphical models with above pairwise potentials are called {\em
%independent-set} or {\em hard-core} models since the support of the
%distribution is on the set of independent sets of
%$\graph$.\footnote{A subset $I \subset V$ is called an independent
%set of $\graph=(V,E)$ if for any $\{u,v\} \in I$, $\{u,v\}
%\notin E$.}

Finally, some notations. Let $\n(v)$ be the set of neighbors of a
vertex $v\in V$, $d_v:= |\n(v)|$ be the degree of $v \in V$, and
$\Delta:= \max_v d_v$ be the maximum degree in the graph $G$.
Further, we define
$$|\psi|~:=~\max_{(u,v)\in
E,x_u,x_v\in\{0,1\}}\left\{e^{|\ln \psi_v(x_v)|}, e^{|\ln
\psi_{u,v}(x_u,x_v)|}\right\}.$$ 
%Without loss of generality, we assume $|\psi|\geq 2$, where one can appropriately rescale potential functions
%to satisfy this.\footnote{This assumption is merely for simplifying statements of our main results in this paper.} 
We primarily focus on the case $|\psi|<\infty$, which excludes the case
$\psi_{u,v}(\cdot,\cdot)=0$. However, this does not hurt the generality
of the results in this paper too much since one can consider tiny perturbations
to such `zero' potential functions so that the distribution remains almost the same.
%we note that our algorithm and
%its analysis still work even for the case
%$\psi_{u,v}(\cdot,\cdot)=0$ such as the independent set model in
%\cite{ShinBP}. 

\subsection{Belief Propagation}

The BP algorithm has messages $\{m^t_{u\rightarrow
v}(\cdot),m^t_{v\rightarrow u}(\cdot)\}:=\{m^t_{u\rightarrow
v}(x_v), m^t_{v\rightarrow u}(x_u): (u,v) \in E, x_v,x_u \in
\{0,1\}\}$ at the $t$-th iteration %on the both sides of edges
and it updates them as
\begin{eqnarray*}
m^{t+1}_{u\to v}(x_v)&\propto&\sum_{x_u\in\{0,1\}}
\psi_{u,v}(x_u,x_v)\psi_u(x_u)\prod_{w\in\n(u)\backslash
v}m^{t}_{w\to u}(x_u),%\quad\mbox{and}\quad
\end{eqnarray*}
where $\sum_{x_v\in\{0,1\}}m^{t+1}_{u\to v}(x_v)=1$. This is
equivalent to the following updating rule on (reduced) messages $\{
m^{t}_{u\to v},m^t_{v\rightarrow u}\}$.
\begin{eqnarray*}
m^{t+1}_{u\to v} &:=& f_{u\rightarrow
v}\left(\prod_{w\in\n(u)\backslash v}m^{t}_{w\to u}\right),
\end{eqnarray*}
where $m^{t}_{u\to v}:=m^{t}_{u\to v}(1)/m^{t}_{u\to v}(0)$ and the
function $f_{u\rightarrow v}:\mathbb{R}_+\to \mathbb{R}_+$ is
defined as
$$f_{u\rightarrow v}(x)~:=~
\frac{\psi_{u,v}(0,1)\psi_u(0)+\psi_{u,v}(1,1)\psi_u(1) \cdot
x}{\psi_{u,v}(0,0)\psi_u(0)+\psi_{u,v}(1,0)\psi_u(1)\cdot x}.$$
The initial messages at the first iteration can be chosen arbitrarily as  
positive real numbers, where the standard choice is 
$m^{1}_{u\to v}=m^{1}_{v\to u}=1$ for all $(u,v)\in E$.

Now the {\em BP fixed-point} of messages $\{m_{u\rightarrow
v},m_{v\rightarrow u}\}$ can be naturally defined as
\begin{eqnarray}
m_{u\to v}~=~ f_{u\rightarrow v}\left(\prod_{w\in\n(u)\backslash
v}m_{w\to u}\right),\qquad\forall~(u,v)\in E.\label{eq:bpfix}
\end{eqnarray}
If $|\psi|<\infty$, one can easily argue the existence of such a (finite) fixed-point using
the Brouwer's fixed-point theorem. This motivates the following notion
of $\varepsilon$-approximate BP fixed-point.
\begin{definition}
The set of messages $\{m_{u\rightarrow v},m_{v\rightarrow u}: (u,v)
\in E\}$ is called an $\varepsilon$-approximate BP fixed-point if
\begin{eqnarray}
\left|\frac{m_{u\to v}}{f_{u\rightarrow
v}\left(\prod_{w\in\n(u)\backslash v}m_{w\to
u}\right)}-1\right|~\leq~\varepsilon, \qquad\forall~(u,v)\in
E.\label{eq:bpfix2}
\end{eqnarray}
\end{definition}

The BP estimates for node and edge marginal probabilities based on
messages, denoted by $\tau_v(\cdot), \tau_{u,v}(\cdot)$ for $v \in
V, (u,v)\in E$, are defined as
\begin{eqnarray*}
\tau_v(x_v)&\propto&  \psi_v(x_v)\prod_{u \in \n(v)} m_{u\to
v}(x_v)\\
\tau_{u,v}(x_u,x_v)&\propto&
\psi_{u}(x_u)\psi_{v}(x_v)\psi_{u,v}(x_u,x_v)\left(\prod_{w \in
\n(u)\setminus v} m_{w\to u}(x_u)\right)\left(\prod_{w \in \n(v)\setminus u}
m_{w\to v}(x_v)\right), %\label{eq:bpmarg}
\end{eqnarray*}
where $\sum_{x_v}\tau_v(x_v)=1$ and
$\tau_v(x_v)=\sum_{x_u}\tau_{u,v}(x_u,x_v)$.
%Hence, the reduced
%messages $\{m_{u\rightarrow v},m_{v\rightarrow u}: (u,v) \in E\}$
%suffices to compute the BP estimates (for marginal probabilities)
%defined above.
%\begin{equation}
%\tau_v(1)~\leq~\tau_v(0)\notag%\label{eq:prop1}
%\end{equation}
%\begin{equation}
% \max\{\tau_{u,v}(1,0),
%\tau_{u,v}(0,1)\} ~\leq~\tau_{u,v}(0,0).\label{eq:prop2}
%\end{equation}

\subsection{Bethe Approximation}\label{sec:bethe}

The Bethe approximation \cite{YWY04} is an approximation to the
logarithm of the partition function (i.e., $\ln Z$), given by
\begin{eqnarray*}
%\lefteqn{
\sum_{v \in V} \sum_{x_v} \tau_v(x_v)
\left[ \ln  \psi_v(x_v) - \ln  \tau_v(x_v) \right] %\nonumber \\
  + \sum_{\{u,v\} \in E} \sum_{x_u,x_v}
\tau_{u,v}(x_u,x_v) \Bigg[ \ln  \psi_{u,v}(x_u,x_v)%\notag\\
%&& \qquad\qquad\qquad\qquad\qquad
- \ln
\frac{\tau_{u,v}(x_u,x_v)}{\tau_u(x_u) \tau_v(x_v)}\Bigg],
\end{eqnarray*}
where $\tau_v(\cdot), \tau_{u,v}(\cdot)$ are marginal estimates at a BP fixed-point.
Under the constraints $\sum_{x_v}\tau_v(x_v)=1$ for all $v\in V$ and
$\tau_v(x_v)=\sum_{x_u}\tau_{u,v}(x_u,x_v)$ for all $v\in V, x_v\in\{0,1\}$, this expression can be written as a function of 
 $\by=[y_v,y_{u,v}]$ where $y_v=\tau_v(1)$ and
$y_{u,v}=\tau_{u,v}(1,1)$.

\begin{eqnarray}
F(\by)&:=&\sum_{v\in V} y_v (\ln \psi_v(1) - \ln y_v) +(1-y_v)
(\ln \psi_v(0) - \ln (1-y_v))\notag\\
&&\qquad+\sum_{(u,v)\in E} \Bigg[(1-y_u-y_v+y_{u,v})\left(\ln
\psi_{u,v}(0,0)- \ln
\frac{1-y_u-y_v+y_{u,v}}{(1-y_u)(1-y_v)}\right)\notag\\
%\end{eqnarray}
%\begin{eqnarray}
&&\qquad\qquad\qquad\qquad+(y_u-y_{u,v})\left(\ln \psi_{u,v}(1,0)-
\ln \frac{y_u-y_{u,v}}{y_u(1-y_v)}\right)\notag\\
&&\qquad\qquad\qquad\qquad+(y_v-y_{u,v})\left(\ln \psi_{u,v}(0,1)-
\ln
\frac{y_v-y_{u,v}}{(1-y_u)y_v}\right)\notag\\
&&\qquad\qquad\qquad\qquad+y_{u,v}\left(\ln \psi_{u,v}(1,1)- \ln
\frac{y_{u,v}}{y_uy_v}\right)\Bigg],\label{eq:bethe}
\end{eqnarray}
where $-F$ is called the Bethe free energy function \cite{YWY04}.
 The gradient $\nabla F(\by) = \left[\frac{\partial
F}{\partial y_v},\frac{\partial F}{\partial y_{u,v}}\right]$ can be
obtained as
\begin{eqnarray}
\frac{\partial F}{\partial y_v} %&=& \ln \psi_v(1) - \ln y_v
%-\ln \psi_v(0) + \ln (1-y_v)\notag\\
%&&\qquad+\sum_{u\in\mathcal N(v)} \Bigg[-\ln \psi_{u,v}(0,0)+ \ln
%\frac{1-y_u-y_v+y_{u,v}}{(1-y_u)(1-y_v)}+1-\frac{1-y_u-y_v+y_{u,v}}{1-y_v}\notag\\
%&&\qquad\qquad\qquad\qquad-\frac{y_u-y_{u,v}}{1-y_v}+\ln
%\psi_{u,v}(0,1)- \ln
%\frac{y_v-y_{u,v}}{(1-y_u)y_v}-1+\frac{y_v-y_{u,v}}{y_v} +\frac{y_{u,v}}{y_v}\Bigg]\notag\\
%&=&\Psi^{(v)} +(d_v-1) \ln y_v -(d_v-1)\ln
%(1-y_v)+\sum_{u\in\mathcal N(v)}\ln (1-y_v-y_u+y_{u,v})-\ln
%(y_v-y_{u,v})\\
&=&\Psi^{(v)} + \ln \frac{1-y_v}{y_v}+\sum_{u\in\mathcal N(v)}\ln
\left(\frac{1-y_v-y_u+y_{u,v}}{1-y_v}\cdot
\frac{y_v}{y_v-y_{u,v}}\right)\label{eq:partial1}\\
\frac{\partial F}{\partial y_{u,v}} %&=&\ln \psi_{u,v}(0,0)- \ln
%\frac{1-y_u-y_v+y_{u,v}}{(1-y_u)(1-y_v)}-1-\ln \psi_{u,v}(1,0)+ \ln
%\frac{y_u-y_{u,v}}{y_u(1-y_v)}+1\notag\\
%&&\qquad\qquad\qquad\qquad-\ln \psi_{u,v}(0,1)+ \ln
%\frac{y_v-y_{u,v}}{(1-y_u)y_v}+1+\ln \psi_{u,v}(1,1)-
%\ln \frac{y_{u,v}}{y_uy_v}-1\notag\\
%&=&\Psi^{(u,v)}+ \ln (y_u-y_{u,v}) +\ln (y_v-y_{u,v}) -\ln
%(1-y_u-y_v+y_{u,v})-\ln y_{u,v}\\
&=&\Psi^{(u,v)}+ \ln
\left(\frac{y_u-y_{u,v}}{1-y_u-y_v+y_{u,v}}\cdot\frac{y_v-y_{u,v}}{
y_{u,v}}\right),\label{eq:partial2}
\end{eqnarray}
where
\begin{equation}\label{eq:defPsi}\Psi^{(v)}:=\ln \frac{\psi_v(1)}{\psi_v(0)}+\sum_{u\in\mathcal N(v)} \ln \frac{\psi_{u,v}(0,1)}{\psi_{u,v}(0,0)}
\qquad\mbox{and}\qquad\Psi^{(u,v)}:=\ln
\frac{\psi_{u,v}(0,0)\,\psi_{u,v}(1,1)}{\psi_{u,v}(1,0)\,\psi_{u,v}(0,1)}.\end{equation}

%Now we say $\by^*=[y^*_v,y^*_{u,v}]$ be a $\varepsilon$ gradient
%point (or stationary point) of $F$ if $\|\nabla F(\by^*)\|_2 \leq
%\varepsilon$.
It is known that there is a one-to-one correspondence between BP fixed-points and zero gradient points of $F$. In particular, one can
obtain the following lemma whose proof can be done easily using
the algebraic expressions \eqref{eq:partial1} and
\eqref{eq:partial2} of gradients. 

\begin{lemma}\label{lem:ytom}
Given $\varepsilon\in[0,1)$, suppose $\by=[y_v,y_{u,v}]$ satisfies
$\|\nabla F(\by)\|_\infty \leq\varepsilon$. Then, the set of
messages $\{m_{u\rightarrow v},m_{v\rightarrow u}:
(u,v) \in E\}$ is a $6\varepsilon$-approximate BP fixed-point if it is given as%satisfy %the unique solution of the following
%equations:
%$$f_{v\rightarrow u}\left(\frac{\psi_v(0)}{\psi_v(1)}\cdot\frac1{ m_{u\rightarrow v}}\cdot\frac{y^*_v}{1-y^*_v}\right)
%~=~  m_{v\rightarrow u}.$$
\begin{equation}
{m_{u\to v}}
=\frac{\psi_{u,v}(0,1)}{\psi_{u,v}(0,0)}\cdot\frac{1-y_v-y_u+y_{u,v}}{1-y_v}\cdot
\frac{y_v}{y_v-y_{u,v}}.\label{eq:convert}
\end{equation}
%In above, $f_{v\rightarrow u}^{*}$ is the inverse function of
%$f_{v\rightarrow u}$.
%Then, it follows that
%$$\left|\ln \frac{m_{u\to
%v}}{f_{u\rightarrow v}\left(\prod_{w\in\n(u)\backslash v}m_{w\to
%u}\right)}\right|~\leq~3\varepsilon+2\Delta\eta,
%\qquad\forall~(u,v)\in E.$$
\end{lemma}
The proof of Lemma \ref{lem:ytom} is presented in Appendix
\ref{sec:pfytom}. 
For 
somewhat `non-intuitive' formula \eqref{eq:convert}, 
we also provide some intuition in Appendix \ref{sec:intytom}.

%\vspace{0.1in}
%\newpage
\section{Algorithm for Computing BP Fixed-Points} \label{sec:two}

In this section, we present the main result of this paper, a new
message-passing algorithm for approximating a BP fixed-point. From
the (algebraic) relationship between approximate BP fixed-points and
near-stationary points of the Bethe free energy function $F$ in
Lemma \ref{lem:ytom}, it is equivalent to compute a near-stationary
point $\by$, i.e., $\|\nabla F(\by)\|_2\leq \varepsilon$.

Our algorithm, described next, for finding such a point is
essentially motivated by the standard (projected) gradient-descent
algorithm. The non-triviality (and novelty) lies in our choice of an
appropriate (time-varying) `projection $[\cdot]_*$' with respect to
the (time-varying) `step-size $\frac1{\sqrt{t}}$' at each iteration
and subsequent analysis of the rate of convergence.

%\subsection{Algorithm Description and Running Time}\label{sec:alg}
%Now we describe our algorithm as follows.

\vspace{0.1in} \noindent {\bf Algorithm A} %\vspace{0.05in}
\vspace{0.1in}\hrule%\vspace{0.05in}
\vspace{0.05in}
\begin{itemize}

\item[1.] Algorithm parameters:
$$\varepsilon\in(0,1)\qquad\mbox{and}\qquad
\mb{y}_{V}(t)= \left[\,y_v(t)\in(0,1):v\in V\,\right]\quad\mbox{at the
$t$-th iteration.}$$ Initially, $y_v(1) = 1/2$ for all $v\in V$.

%\vspace{0.05in}

\item[2.] $\by_V(t)$ is updated as:
\begin{eqnarray*}
y_v(t+1) &=&\Bigg[y_v(t)+ \frac1{\sqrt{t}}\Bigg(\Psi^{(v)} + \ln
\frac{1-y_v(t)}{y_v(t)}\\
&&\qquad\qquad\qquad+\sum_{u\in\mathcal N(v)}\ln
\left(\frac{1-y_v(t)-y_u(t)+y_{u,v}(t)}{1-y_v(t)}\cdot
\frac{y_v(t)}{y_v(t)-y_{u,v}(t)}\right)\Bigg)\Bigg]_*,
%\frac1{2^{d+7} (d^2+6d+2) \sqrt{t+1}}
\end{eqnarray*}
where the projection $[\cdot]_*$ at the $t$-th iteration is defined
as $$[x]_*:=\begin{cases} x&\mbox{if}~ \frac{0.1}{t^{1/4}}\leq x\leq
1-\frac{0.1}{t^{1/4}}\\
\frac{0.1}{t^{1/4}}&\mbox{if}~x<\frac{0.1}{t^{1/4}}\\
1-\frac{0.1}{t^{1/4}}&\mbox{if}~x>1-\frac{0.1}{t^{1/4}}\end{cases},$$ and
$y_{u,v}(t)>0$ is computed as the unique solution satisfying
\begin{eqnarray*}
e^{\Psi^{(u,v)}}\cdot\frac{y_u(t)-y_{u,v}(t)}{1-y_u(t)-y_v(t)+y_{u,v}(t)}\cdot\frac{y_v(t)-y_{u,v}(t)}{
y_{u,v}(t)}=1\quad\mbox{and}\quad
y_{u,v}(t)<\min\{y_v(t),y_u(t)\}.\end{eqnarray*}

\item[3.] Compute messages $\{m_{u\rightarrow v},m_{v\rightarrow u}\}$ as
%the unique solution satisfying
%$$\frac{y_v(t)}{1-y_v(t)}~=~\frac{\psi_v(1)}{\psi_v(0)}\cdot m_{u\rightarrow v} \cdot f_{v\rightarrow u}^{*}(m_{v\rightarrow u}).$$
%$$f_{v\rightarrow u}\left(\frac{\psi_v(0)}{\psi_v(1)}\cdot\frac1{ m_{u\rightarrow v}}\cdot\frac{y_v(t)}{1-y_v(t)}\right)
%~=~  m_{v\rightarrow u}.$$
$$m_{u\to v}~=~\frac{\psi_{u,v}(0,1)}{\psi_{u,v}(0,0)}\cdot\frac{1-y_v(t)-y_u(t)+y_{u,v}(t)}{1-y_v(t)}\cdot
\frac{y_v(t)}{y_v(t)-y_{u,v}(t)}.$$

\item[4.] Terminate if $\{m_{u\rightarrow v},m_{v\rightarrow u}\}$
is an $\varepsilon$-approximate BP fixed-point.
\end{itemize}
\hrule \vspace{0.2in}

%In the next section, we will prove that $\by(t)$ is always in
%$(0,1)^n$ with an appropriate choice of $\alpha_0>0$, which is
%important to guarantee the validity of computations in the algorithm
%e.g., computing logarithms.
The algorithm is clearly implementable through message-passing where
each node $u$ sends $y_u(t)$ to all of its neighbors $v \in \n(u)$
at each iteration. We also note that solving the second step for
computing $y_{u,v}(t)$ can be done efficiently since it is solving a
quadratic equation whose coefficients are decided by $y_v(t)$ and
$y_u(t)$. %The equations in the third step is also easy to solve
%because for each $(u,v)\in E$, two messages $\{m_{u\rightarrow
%v},m_{v\rightarrow u}\}$ is the unique solution of two equations in
%terms of $y_v(t)$ and $y_u(t)$. Further, it is easy to see that they
%are linear equations with respect to appropriate M{\"o}bius
%transformations of $m_{u\rightarrow v}$ and $m_{v\rightarrow u}$.

This algorithm has several variations:
%There are several variations of choice for the algorithm. 
\begin{itemize}
	\item[$\circ$] The initial
value for $[\by_V(1)]$ can be chosen arbitrarily as other values, e.g., $y_v(1)=0.7$.
 \item[$\circ$] The step-size $\frac1{\sqrt{t}}$ can be replaced by any quantity of the same order, e.g.,
$\frac{0.1}{\sqrt{t}},\frac1{\sqrt{t+100}}$.
 \item[$\circ$] $\frac{0.1}{t^{1/4}}$ in the the projection can be replaced by any quantity of same order, e.g.,  
$\frac{1}{t^{1/4}+100}$.
\end{itemize}
In particular, we recommend to use a smaller step-size than $\frac1{\sqrt{t}}$ (such as $\frac1{\sqrt{t+100}}$)
for practical purposes. With these variations, the algorithm may find a different approximate BP fixed-point, but
the following running time guarantee of the algorithm always holds.

%We establish the following running time of the algorithm.
\begin{theorem}\label{thm:runningtime}
%Consider the following constants $\varepsilon_1$, $\varepsilon_2$
%and $\varepsilon_3$.
%$$\varepsilon_1:=\frac{\varepsilon_2^\Delta}{2\left(|\psi|^{2(\Delta+1)}+\Delta\varepsilon_2^{\Delta-1}\right)}\qquad
%\varepsilon_2=\frac{1/2}{|\psi|^{2(\Delta+1) }+\Delta+1}\qquad
%\varepsilon_3:=\frac{\varepsilon_1}{2 w_{\max}}.$$
%If the algorithm runs with
%$$\alpha_0=\frac{\delta}{4\ln \frac{1}{2\delta}}
%\qquad\mbox{and}\qquad\delta=\frac{1/2}{2(\Delta+1)|\psi|^{4\Delta+1}+1},$$
%then it runs well: at every iteration,
%$$\by(t)\in (0,1)^n,\qquad\forall~t\geq0.$$
Algorithm A terminates in ${(|\psi|+2)}^{O(\Delta)} n^2 \varepsilon^{-4}\log^3
(n \varepsilon^{-1})$ iterations.
\end{theorem}
For the reader's convenience, we recall the definitions of symbols used in the above theorem.
\begin{itemize}
	\item[$\circ$] $|\psi|$ is for the range of potential functions, i.e.,
	$$|\psi|~:=~\max_{(u,v)\in
	E,x_u,x_v\in\{0,1\}}\left\{e^{|\ln \psi_v(x_v)|}, e^{|\ln
	\psi_{u,v}(x_u,x_v)|}\right\}.$$
	\item[$\circ$] $n$ is the number of nodes and $\Delta$ is the maximum degree in the underlying graph.
	\item[$\circ$] $\varepsilon$ is a parameter which decides the quality of the produced approximate BP fixed-point.
\end{itemize}
	 % as long as $|\psi|=O(1)$.
The proof of Theorem \ref{thm:runningtime} is presented in the
following section. Note that the algorithm may require to maintain
irrational messages or rational messages of long bits. In Section
\ref{sec:fptas}, we present a minor modification of the algorithm to
fix the issue, which leads to a fully polynomial-time approximation
algorithm (FPTAS) to compute an approximate BP fixed-point.

\subsection{Proof of Theorem
\ref{thm:runningtime}}\label{sec:mainpf}
\renewcommand{\y}{\by}

We first define $F^*$ on $(0,1)^n$: for $\by_V=[y_v] \in (0,1)^n$, let
$$F^*(\by_V)~:=F(\by_V,\by_E),$$
where $F$ is the (original) Bethe free energy function defined in
\eqref{eq:bethe} and the additional vector
$\by_E=[y_{u,v}]\in(0,1)^{|E|}$ is defined as the solution
satisfying
\begin{equation} \Psi^{(u,v)}+ \ln
\left(\frac{y_u-y_{u,v}}{1-y_u-y_v+y_{u,v}}\cdot\frac{y_v-y_{u,v}}{
y_{u,v}}\right)=0\qquad\mbox{and}\qquad
y_{u,v}<\min\{y_v,y_v\}.\label{eq:secondmar}\end{equation} 
Observe
that each $y_{u,v}$ is a function of $y_u,y_v$, i.e.,
$y_{u,v}=y_{u,v}(y_u,y_v)$, and
we can write $F^*(\by_V)=F(\by_V,\by_E(\by_V))$.
One can check that the gradient of $F^*$ is given by \eqref{eq:derivform}, which
is the same as that of $F$ in \eqref{eq:partial1}.
\begin{equation}\label{eq:derivform}
	\frac{\partial F^*}{\partial y_v}~=~
\Psi^{(v)} + \ln \frac{1-y_v}{y_v}+\sum_{u\in\mathcal N(v)}\ln
\left(\frac{1-y_v-y_u+y_{u,v}}{1-y_v}\cdot
\frac{y_v}{y_v-y_{u,v}}\right),\qquad\forall~v\in V,\end{equation} 
where we
recall that $y_{u,v}$ is decided in terms of $y_u,y_v$ from
\eqref{eq:secondmar}. This implies that the updating procedure of
$\by_V(t)$ in the algorithm is simply 
\begin{equation}
\by_V(t+1)~:=~\left[\by_V(t)+\frac1{\sqrt{t}}\,\nabla
F^*(\by_V(t))\right]_*.\label{eq:key}
\end{equation}

Based on this interpretation, we start to prove the running time of
the algorithm by stating the following key lemma.
\begin{lemma}\label{lem:key}
Define $\delta>0$ as the largest real number that satisfies the
following conditions.
$$\delta~\leq~ \frac{1/2}{|\psi|^{6\Delta+2}+1}
\qquad\mbox{and}\qquad \delta\ln\frac1{\delta}~\leq~ \frac1{400}.
$$
Then, it follows that
$$\by_V(t)\in D:=[\delta,1-\delta]^n,\qquad \forall~t\geq t_*:=\frac{0.0001}{\delta^4}.%\left(4\frac1{\delta}
%\ln\frac1{2\delta}\right)^{2}.
$$
\end{lemma}
%\subsection{Proof of Correctness}\label{sec:pfcor}
\begin{proof}
First observe that $\by_V(t_*)\in D$ due to our choice of projection
$[\cdot]_*$ and $\frac{0.1}{t_*^{1/4}}=\delta$. Hence, it suffices to
establish the following three steps: for all $v\in V$ and $t> t_*$,
%$(u,v)\in \edges$,
\begin{align}
\frac{\partial F^*}{\partial y_v}~\leq~ 0 &\qquad\mbox{if}~~ y_v\geq 1- 2\delta~~\mbox{and}~~ \y\in D,\label{eq:step0}\\
\frac{\partial F^*}{\partial y_v}~\geq~ 0 &\qquad\mbox{if}~~ y_v\leq 2\delta~~\mbox{and}~~ \y\in D,\label{eq:step1}\\
\frac1{\sqrt{t}}\left|\frac{\partial F^*}{\partial y_v}\right|~\leq~
\frac{\delta}2 &\qquad\mbox{if}~~\y\in D.\label{eq:step2}
%\frac{1}{2^{d+7}(d^2+6d+2)},\notag
\end{align}
From \eqref{eq:key}, \eqref{eq:step0}, \eqref{eq:step1} and
\eqref{eq:step2}, it clearly follows that $\by_V(t)\in D$ for all
$t\geq t_*$.\footnote{In fact, one can show stronger inequalities: $3\delta/2\leq y_v(t)\leq 1- 3\delta/2$ for all $v\in V, t\geq t^*$. }

\paragraph{Proof of \eqref{eq:step0}.} We first provide a proof of
\eqref{eq:step0}. To this end, if $\by_V=[y_v]\in (0,1)^n$, we have
\begin{eqnarray}
\frac{1-y_v-y_u+y_{u,v}}{1-y_v}\cdot \frac{y_v}{y_v-y_{u,v}} &=&
\frac1{1+\frac{y_u-y_{u,v}}{1-y_v-y_u+y_{u,v}}} \cdot
\frac{y_v}{y_v-y_{u,v}}\notag\\
&=&\frac1{1+e^{-\Psi^{(u,v)}}\cdot\frac{y_{u,v}}{y_v-y_{u,v}}}\cdot
\left(1+\frac{y_{u,v}}{y_v-y_{u,v}}\right)\notag\\
&\leq&\max\left\{1,e^{\Psi^{(u,v)}}\right\}\notag\\
&\leq&|\psi|^4,\label{eq1}
\end{eqnarray}
where we use the definition \eqref{eq:secondmar} of $y_{u,v}$. Using
this, \eqref{eq:step0} follows from
\begin{eqnarray*}
\frac{\partial F^*}{\partial y_v}&=& \Psi^{(v)} + \ln
\frac{1-y_v}{y_v}+\sum_{u\in\mathcal N(v)}\ln
\left(\frac{1-y_v-y_u+y_{u,v}}{1-y_v}\cdot
\frac{y_v}{y_v-y_{u,v}}\right)\\
&\leq&2(\Delta+1)\ln |\psi| +\ln
\frac{2\delta}{1-2\delta}+\Delta\ln |\psi|^4\\
&=&\ln\frac{|\psi|^{6\Delta+2}}{\frac1{2\delta}-1}\\
&\leq& 0,
\end{eqnarray*}
where the last inequality follows from our choice of
$\delta\leq\frac{1/2}{|\psi|^{6\Delta+2}+1}$.

\paragraph{Proof of \eqref{eq:step1}.} Second, we provide a proof of
\eqref{eq:step1}. Similarly as we did in \eqref{eq1}, we have
\begin{eqnarray}
\frac{1-y_v-y_u+y_{u,v}}{1-y_v}\cdot \frac{y_v}{y_v-y_{u,v}}
&\geq&\min\left\{1,e^{\Psi^{(u,v)}}\right\}
~\geq~\frac1{|\psi|^4}.\label{eq2}
\end{eqnarray}
Hence, \eqref{eq:step1} follows as
\begin{eqnarray*}
\frac{\partial F^*}{\partial y_v}&=& \Psi^{(v)} + \ln
\frac{1-y_v}{y_v}+\sum_{u\in\mathcal N(v)}\ln
\left(\frac{1-y_v-y_u+y_{u,v}}{1-y_v}\cdot
\frac{y_v}{y_v-y_{u,v}}\right)\\
&\geq&-2(\Delta+1)\ln |\psi| +\ln
\frac{1-2\delta}{2\delta}+\Delta\ln \frac1{|\psi|^4}\\
&=&\ln\frac{\frac1{2\delta}-1}{|\psi|^{6\Delta+2}}\\
&\geq& 0,
\end{eqnarray*}
where the last inequality follows again from our choice of
$\delta\leq\frac{1/2}{|\psi|^{6\Delta+2}+1}$.

\paragraph{Proof of \eqref{eq:step2}.} Finally, we provide a proof of
\eqref{eq:step2}. Using \eqref{eq1} and \eqref{eq2}, we obtain
\begin{eqnarray*}
\left|\frac{\partial F^*}{\partial y_v}\right|&\leq&
\left|\Psi^{(v)}\right| + \left|\ln
\frac{1-y_v}{y_v}\right|+\sum_{u\in\mathcal N(v)}\left|\ln
\left(\frac{1-y_v-y_u+y_{u,v}}{1-y_v}\cdot
\frac{y_v}{y_v-y_{u,v}}\right)\right|\\
&\leq&2(\Delta+1)\ln |\psi| +\ln
\frac{1-2\delta}{2\delta}+\Delta\ln |\psi|^4\\
&\leq&\ln \frac{1-\delta}{\delta}+\ln \frac{1-2\delta}{2\delta}\\
&\leq&2\ln \frac{1}{\delta}\\
&\leq &\frac{\delta}{2}\cdot \sqrt{t_*},
\end{eqnarray*}
where the last inequality follows from our choices of $\delta,t_*$ which
imply
$$\sqrt{t_*}~=~\frac{0.01}{\delta^{2}}~\geq~4\frac1{\delta}
\ln\frac1{\delta}.$$ Therefore,
$\frac1{\sqrt{t}}\left|\frac{\partial F^*}{\partial y_v}\right|\leq \frac{\delta}{2}\cdot \sqrt{\frac{t_*}{t}}\leq \frac{\delta}{2}$
for $t\geq t^*$.
This completes the proof of Lemma
\ref{lem:key}.
\end{proof}
%\subsection{Proof of Running Time}\label{sec:pfrun}

Using the above lemma, we will obtain the running-time guarantee of Algorithm A. %terminates in
%$\frac{{(|\psi|+2)}^{O(\Delta)} n^2 \log^3 (n \varepsilon^{-1})}{\varepsilon^4}$
%iterations until it outputs an $\varepsilon$-approximate BP fixed-point. 
We first explain why it suffices to show the following:
\begin{equation}\label{eq:final}
\sum_{t=t_*}^T c_t\cdot \|\nabla F^* (\y_V(t))\|_2^2
=\frac{({|\psi|+2)}^{O(\Delta)}\,n\,\log T}{\sqrt{T}},\quad
\mbox{where} \quad c_t=c_t(t^*,T):=\frac{t^{-1/2}}{\sum_{t=t^*}^T t^{-1/2}}. 
\end{equation}
The above
equality suggests that we can choose $T={{(|\psi|+2)}^{O(\Delta)} n^2{\varepsilon^{-4}} \log^3 (n \varepsilon^{-1})}$ such that
\begin{equation*}
\sum_{t=t_*}^T c_t\cdot \|\nabla F^* (\y_V(t))\|_2^2\leq
\left(\frac{\varepsilon}6\right)^2.
\end{equation*}
From $\sum_{t=t_*}^T c_t=1$, there exists $t\in[t_*,T]$ such that
$\|\nabla F^* (\y_V(t))\|_2\leq \varepsilon/6$. Further, we have
%if $\by_E(t)$ is defined from \eqref{eq:secondmar}, then
$$\|\nabla F
(\y_V(t),\by_E(t))\|_2~=~ \|\nabla F^* (\y_V(t))\|_2~\leq~
\varepsilon/6,$$ 
where the first equality follows from
\begin{eqnarray*}
\frac{\partial F}{\partial y_v}(\y_V(t),\by_E(t)) = 
\frac{\partial F^*}{\partial y_v}(\y_V(t))\quad\mbox{from}~~\eqref{eq:derivform}\quad\mbox{and}\quad
\frac{\partial F}{\partial y_{u,v}}(\y_V(t),\by_E(t)) = 
0\quad\mbox{from}~~\eqref{eq:secondmar}.
\end{eqnarray*}
Then, Lemma \ref{lem:ytom} implies that the
computed messages at the $t$-th iteration is an
$\varepsilon$-approximate BP fixed-point since $\|\cdot\|_\infty \leq \|\cdot\|_2$.

Now we proceed toward establishing the desired equality
\eqref{eq:final}. The important implication of Lemma \ref{lem:key}
is that the algorithm does not need the projection $[\cdot]_*$ after
the $t_*$-th iteration. In other words, from 
\eqref{eq:key}, we have that
$$\y_V(t+1)~=~\y_V(t)+ \frac1{\sqrt{t}}\,\nabla F^* (\y_V(t)),\qquad\forall~t\geq t_*.$$
In what follows, we will assume $t\geq t_*$ and $\by_V(t)\in D$, 
where $t_*$ and $D$ are defined in Lemma \ref{lem:key}.

%Using the Taylor's expansion, 
Doing a Taylor series expansion  of $F^*$ around $\y_V(t)$,
we have
\begin{eqnarray}
F^*(\y_V(t+1))&=&F^*\left(\y_V(t)+ \frac1{\sqrt{t}}\,\nabla F^* (\y_V(t))\right)\nonumber\\
&=&F^*(\y_V(t))+\nabla F^*(\y_V(t))'\cdot \frac1{\sqrt{t}}\,\nabla F^*
(\y_V(t))\nonumber\\
&&\qquad\qquad +\frac12\frac1{\sqrt{t}}\,\nabla F^* (\y_V(t))'\cdot R \cdot
\frac1{\sqrt{t}}\,\nabla F^* (\y_V(t)),\label{eq:taylor1}
\end{eqnarray}
where $R$ is a $n\times n$ matrix such that
$$|R_{vw}|\leq \sup_{[y_v]\in B}\left|\frac{\partial^2 F^*}{\partial y_v \partial y_w}\right|,$$
and $B$ is a $L_{\infty}$-ball in $\R^{n}$ centered at $\y_V(t)\in D$
with its radius
$$r=\max_{v\in V}\left|\frac1{\sqrt{t}}\,\frac{\partial F^*}{\partial y_v}(\y(t))\right|.$$
From \eqref{eq:step2}, we know that $r\leq \frac{\delta}2$. Hence,
$y_v\in \left[\delta/2,1-\delta/2\right]$ for $\y_V=[y_v]\in B$. Using
this with $\delta={(|\psi|+2)}^{-O(\Delta)}$, one can check that
$$\sup_{\y\in B}\left|\frac{\partial^2
F^*}{\partial y_v
\partial y_w}\right|~=~
\begin{cases}
{(|\psi|+2)}^{O(\Delta)}&\mbox{if}~v=w~\mbox{or}~ (v,w)\in \edges\\
\quad 0&\mbox{otherwise} \end{cases}.$$

 Therefore, using these
bounds, the equality \eqref{eq:taylor1} reduces to
\begin{align*}
F^*(\y_V(t+1))
%&~\geq~ F^*(\y(t))+\frac1{\sqrt{t}}\,\|\nabla F^* (\y(t))\|_2^2-\alpha^2(t)\,O\left(|\edges|d^5\,2^d\right)\nonumber\\
&~\geq~F^*(\y_V(t))+\frac1{\sqrt{t}}\,\|\nabla F^*
(\y_V(t))\|_2^2-\frac{O\left({(|\psi|+2)}^{O(\Delta)}|E|\right)}t\\%\label{eq:taylor2}\\
&~\geq~F^*(\y_V(t))+\frac1{\sqrt{t}}\,\|\nabla F^*
(\y_V(t))\|_2^2-\frac{O\left({(|\psi|+2)}^{O(\Delta)}n\right)}t,%\label{eq:taylor2}
\end{align*}
since $|E|\leq \Delta \cdot n$. If we sum the above inequality over
$t$ from $t_*$ to $T-1$, we have
\begin{eqnarray*}
F^*(\y_V(T))&=&F^*(\y_V(t_*)) + \sum_{t=t_*}^{T-1} F^*(\y_V(t+1)) - F^*(\y_V(t)) \\
&\geq&F^*(\y_V(t_*))+\sum_{t=t_*}^{T-1}\frac1{\sqrt{t}}\,\|\nabla
F^*
(\y_V(t))\|_2^2-O\left({(|\psi|+2)}^{O(\Delta)}n\right)\,\sum_{t=t_*}^{T-1}\frac1t.%\label{eq:taylor3}
\end{eqnarray*}
Since $|F^*(\y_V)|=O(\Delta n)$ for $\y_V\in D$, we obtain
%\begin{eqnarray*}
$$\sum_{t=t_*}^{T-1}\frac1{\sqrt{t}}\,\|\nabla F^* (\y_V(t))\|_2^2~\leq~ O( \Delta n)+
O\left({(|\psi|+2)}^{O(\Delta)}n\right)\,\sum_{t=t_*}^{T-1}\frac1{t}.$$%\label{eq:taylor4}
%\end{eqnarray*}
Thus, we finally obtain the desired conclusion \eqref{eq:final} as
follows.
\begin{align*}
\sum_{t=t_*}^T c_t \cdot \|\nabla F^* (\y_V(t))\|_2^2
&=\frac1{\sum_{t=t_*}^{T}\frac1{\sqrt{t}}}\sum_{t=t_*}^{T}\frac1{\sqrt{t}}\,\|\nabla F^* (\y_V(t))\|_2^2\\
&\leq\frac1{\sum_{t=t_*}^{T}\frac1{\sqrt{t}}}\left(O(\Delta n)+ O\left({(|\psi|+2)}^{O(\Delta)}n\right)\,\sum_{t=t_*}^{T}\frac1{t}\right)\\
%&\stackrel{(a)}{=}O\left(\frac{1}{\sqrt{T}}\right)
%\left(O(n)+ O\left(n\right)\log T\right)\\
&=\frac{{(|\psi|+2)}^{O(\Delta)} n \log T}{\sqrt{T}},
\end{align*}
where we recall that $t_*=O(\delta^{-4})={(|\psi|+2)}^{O(\Delta)}$. This completes
the proof of Theorem \ref{thm:runningtime}.

\subsection{Modification to FPTAS}\label{sec:fptas}

In this section, we provide a minor modification of Algorithm A in
Section \ref{sec:two}, to establish a fully polynomial-time
approximation scheme (FPTAS) for the BP fixed-point computation. We
will show a polynomial number of significant bits in each message $y_v(t)$
is enough for the performance of the algorithm. 
To this end, we define the following
function $g^{t}=[g_v^{t}]$ which describes the updating rule (at the
$t$-th iteration) of Algorithm A, i.e.,
$$\by_V(t+1)~=~g^{t}(\by_V(t))\qquad\mbox{under Algorithm A}.$$ 
We propose the following
algorithm.

\vspace{0.1in} \noindent {\bf Algorithm B}
\vspace{0.1in}\hrule%\vspace{0.05in}
\vspace{0.05in}
\begin{itemize}

\item[1.] Algorithm parameters:
$$k\in\mathbb N,\qquad \varepsilon\in(0,1)\qquad\mbox{and}\qquad
\bz(t)= \left[\,z_v(t)\in(0,1):v\in V\,\right]\quad\mbox{at the
$t$-th iteration,}$$ %Initially, $y_v(0) = 1/2$ for all $v\in V$.
where $z_v(1)=1/2$ and $z_v(t)$ has $k$ bits (i.e., $2^k z_v(t)\in \mathbb{Z}$) for
all $t\geq1$, $v\in V$.
%\vspace{0.05in}

\item[2.] $\bz(t)$ is updated as:
\begin{eqnarray*}
&&\Bigg|z_v(t+1) - g^{t}_v(\bz(t))\Bigg|~\leq~\frac1{2^k}.
%\frac1{2^{d+7} (d^2+6d+2) \sqrt{t+1}}
\end{eqnarray*}

\item[3.] Compute a set of messages $\{m_{u\rightarrow v},m_{v\rightarrow u}\}$
satisfying
%$$\frac{z_v(t)}{1-z_v(t)}~=~\frac{\psi_v(1)}{\psi_v(0)}\cdot m_{u\rightarrow v} \cdot f_{v\rightarrow u}^{*}(m_{v\rightarrow u}).$$
%$$f_{v\rightarrow u}\left(\frac{\psi_v(0)}{\psi_v(1)}\cdot\frac1{ m_{u\rightarrow v}}\cdot\frac{z_v(t)}{1-z_v(t)}\right)
%~=~  m_{v\rightarrow u}.$$
$${m_{u\to v}}~=~{\frac{\psi_{u,v}(0,1)}{\psi_{u,v}(0,0)}\cdot\frac{1-z_v(t)-z_u(t)+z_{u,v}(t)}{1-z_v(t)}\cdot
\frac{z_v(t)}{z_v(t)-z_{u,v}(t)}},$$ where $z_{u,v}(t)>0$ is
computed to satisfy
\begin{eqnarray*}
\left|\Psi^{(u,v)}+ \ln
\left(\frac{z_u(t)-z_{u,v}(t)}{1-z_u(t)-z_v(t)+z_{u,v}(t)}\cdot\frac{z_v(t)-z_{u,v}(t)}{
z_{u,v}(t)}\right)\right|~\leq~ \frac{\varepsilon}6.\end{eqnarray*}

\item[4.] Terminate if $\{m_{u\rightarrow v},m_{v\rightarrow u}\}$
is an $\varepsilon$-approximate BP fixed-point.
\end{itemize}
\hrule \vspace{0.2in} We note that each step in the above algorithm
is executable in a polynomial number of bitwise operations with
respect to ${\Delta}$, $1/\varepsilon$, $k$ and $n$. Step 2
to compute $g_v^t$ consists of $O(\Delta)$ arithmetic operations,
logarithm, division, addition, square root and multiplication.
Furthermore, the equations in Step 3 to compute
$\{m_{u\rightarrow v},m_{v\rightarrow u}\}$ can be solved in a
polynomial number of bitwise operations with respect to
$1/\varepsilon$ and $k$.

Now we state the following theorem, which shows that one can choose
$k$ as a polynomial in terms of $n$, $1/\varepsilon$ and
${|\psi|}^{\Delta}$. This implies that Algorithm B is an FPTAS for such a
choice of $k$ as long as $\Delta=O(\log n)$ and $|\psi|=O(1)$. We note that one can
obtain the explicit bound of $k$ in terms of $\Delta$, $|\psi|$, $n$
and $\varepsilon$ via explicitly calculating each step in our
proof.\footnote{Another naive way to avoid such an explicit choice
of $k$ is to run Algorithm B `polynomially' many times by increasing
$k$ (as well as the number of iterations) until it succeeds.}
\begin{theorem}\label{thm:final}
There exists a $k={(|\psi|+2)}^{O(\Delta)}n^2
\varepsilon^{-4} \log^4 (n \varepsilon^{-1})$ such that Algorithm B
terminates in ${(|\psi|+2)}^{O(\Delta)}n^2 \varepsilon^{-4} \log^3 (n
\varepsilon^{-1})$ iterations.
\end{theorem}
\begin{proof}
To begin with, one can check that $g^t$ is ${(|\psi|+2)}^{O(\Delta)}$-Lipschitz
in $[\delta/2,1-\delta/2]^n$ where $\delta={(|\psi|+2)}^{-O(\Delta)}$ is the
constant in Lemma \ref{lem:key}. Formally speaking, for all $t\geq
1$, $\by_1,\by_2\in [\delta/2,1-\delta/2]^n$,
$$\|g^t(\by_1) - g^t(\by_2)\|_1~\leq~L\cdot \|\by_1-\by_2\|_1,$$
where $L={(|\psi|+2)}^{O(\Delta)}$. Let $\by_V(t)$ and $\bz(t)$ be variables of
Algorithm A and B, respectively. Initially, $\by_V(1)=\bz(1)$.
Then, if $\by_V(t),\bz(t)\in [\delta/2,1-\delta/2]^n$,
\begin{eqnarray}
\|\by_V(t+1)-\bz(t+1)\|_1&\leq&\|g^t(\by_V(t))-g^t(\bz(t))\|_1 +
\frac{n}{2^k}\notag\\
&\leq&L\cdot \|\by_V(t)-\bz(t)\|_1 + \frac{n}{2^k}\notag\\
 &=&h\left(\|\by_V(t)-\bz(t)\|_1\right),\label{eq3}
\end{eqnarray}
where we define $h(x):=L \cdot x + \frac{n}{2^k}$. From Lemma 3, we
know that $\by_V(t)\in [\delta,1-\delta]^n$ for all $t\geq1$. Hence,
%\eqref{eq3} holds 
for $t\geq 1$ with $h^{(t-1)}(0)\leq
\delta/2$,\footnote{$h^{(t)}$ is the function composing $h$ `$t$
times', i.e., $h^{(t)}=h\circ h^{(t-1)}$ and $h^{(1)}=h$.} we have that $\by_V(t),\bz(t)\in [\delta/2,1-\delta/2]^n$ and
$$\|\by_V(t+1)-\bz(t+1)\|_1~\leq~
 h\left(\|\by_V(t)-\bz(t)\|_1\right)~\leq~ h^{(t)}\left(\|\by_V(1)-\bz(1)\|_1\right) ~=~h^{(t)}(0).$$
Further,
using $h^{(t)}(0)<(L+1)^t\cdot n/2^k $, it follows that
\begin{eqnarray}
\|\by_V(t)-\bz(t)\|_1\leq \gamma,\qquad\mbox{for}~~t\leq T:=\frac{\log
\left(2^{k}\gamma/n \right)}{\log (L+1)}+1,\label{eq4}
\end{eqnarray}
where $\gamma< \delta/2$ will be specified later on.

%Now one can choose $k=2^{(\Delta)}n^2 \gamma^{-4} \log^4 (n
%\gamma^{-1})$ so that $\|\nabla
%F^*(\by(t))\|_\infty\leq\varepsilon/10$ in $T$ iterations under the
%algorithm A. Hence, in the same number of iterations,

Now it is not hard to see that for $t\leq T$ (i.e.,
$\by_V(t),\bz(t)\in [\delta/2,1-\delta/2]^n$ and
$\|\by_V(t)-\bz(t)\|_1\leq \gamma$), if $\|\nabla
F^*(\by_V(t))\|_\infty\leq\gamma$, then $\|\nabla
F^*(\bz(t))\|_\infty\leq\gamma\cdot {(|\psi|+2)}^{O(\Delta)}$. Observe that one
can choose $\gamma=\varepsilon/{(|\psi|+2)}^{O(\Delta)}$ and
$k={(|\psi|+2)}^{O(\Delta)}n^2 \gamma^{-4} \log^4 (n \gamma^{-1})$ so that
Algorithm A has $\|\nabla F^*(\by_V(t))\|_\infty\leq\gamma$ in $T$
iterations, and hence Algorithm B has $\|\nabla
F^*(\bz(t))\|_\infty\leq\varepsilon/6$ in the same number of
iterations. Therefore, the conclusion of Theorem \ref{thm:final}
follows from Lemma \ref{lem:ytom}.
\end{proof}

\section{Extension to Non-Binary Graphical Models}\label{sec:three}

In this section, we discuss how to design a similar algorithm to
those in the previous section for non-binary graphical models. Here
we provide a high-level description, but one can check the further
details based on the identical arguments to the binary case in
Section \ref{sec:two}.

Consider non-binary random variables $\{x_v\}$ in the graphical
model described in Section \ref{sec:one}, i.e., $x_v\in
[Q]=\{0,1,\dots,Q-1\}$ for some $Q\geq 3$. Hence, the potential
functions $\psi_{u,v}$ and $\psi_v$ are functions on $[Q]^2$ and
$[Q]$, respectively. We remind the reader that the essential goal is
to find a near-stationary point of the following Bethe approximation
under the constraints $\sum_{x_v}\tau_v(x_v)=1$ for all $v\in V$ and
$\tau_v(x_v)=\sum_{x_u}\tau_{u,v}(x_u,x_v)$ for all $v\in V, x_v\in [Q]$.
\begin{eqnarray*}
%\lefteqn{
\sum_{v \in V} \sum_{x_v} \tau_v(x_v)
\left[ \ln  \psi_v(x_v) - \ln  \tau_v(x_v) \right] %\nonumber \\
  + \sum_{\{u,v\} \in E} \sum_{x_u,x_v}
\tau_{u,v}(x_u,x_v) \Bigg[ \ln  \psi_{u,v}(x_u,x_v)%\notag\\
%&& \qquad\qquad\qquad\qquad\qquad
- \ln \frac{\tau_{u,v}(x_u,x_v)}{\tau_u(x_u) \tau_v(x_v)}\Bigg].
\end{eqnarray*}

First, for every pair $(p,q)\in [Q]^2$ with $p\neq q$, one can define $F^*_{p,q}$
on $\by_V=[y_v]\in[0,1]^n$ similar to $F^*$ in Section
\ref{sec:mainpf}, by (a) fixing variables except for
$\left\{\tau_v(x_v):x_v\in\{p,q\},v\in V\right\}$ and
$\left\{\tau_{u,v}(x_u,x_v):x_v,x_u\in\{p,q\},(u,v)\in E\right\}$, (b) setting
$y_v=\tau_v(p)$, and (c) considering the constraints
$\sum_{x_v}\tau_v(x_v)=1$ and
$\tau_v(x_v)=\sum_{x_u}\tau_{u,v}(x_u,x_v)$. Now the algorithm for a
non-binary graphical model maintains variables
$$\left\{\tau^t_v(x_v):v\in V, x_v\in [Q]\right\}\qquad\mbox{and}\qquad
\left\{\tau_{u,v}^t(x_u,x_v):(u,v)\in E, (x_u,x_v)\in
[Q]^2\right\},$$ at the $t$-th iteration. At each round, it picks a
pair $(p,q)\in[Q]^2$ with $p\neq q$ in a round-robin fashion and updates the vector
$\left\{\tau_v^t(p)\right\}$ as
$$\tau_v^{t+1}(p)~=~
\left[\tau_v^t(p)+\frac{1}{\sqrt{t}}\nabla
F^*_{p,q}\left(\left\{\tau_u^t(p):u\in V\right\}\right)\right]_*.$$
This is equivalent to
\begin{eqnarray}
\tau_v^{t+1}(p) &= &\Bigg[\tau_v^{t}(p)+ \frac1{\sqrt{t}}\Bigg(\Psi^{(v)}_{p,q} + \ln
\frac{c_{p,q}^t-\tau_v^{t}(p)}{\tau_v^{t}(p)}\nonumber\\
&&\qquad\qquad\qquad+\sum_{u\in\mathcal N(v)}\ln
\left(\frac{c_{p,q}^t-\tau_v^{t}(p)-\tau_u^{t}(p)+\tau_{u,v}^{t}(p,p)}{c_{p,q}^t-\tau_v^{t}(p)}\cdot
\frac{\tau_v^{t}(p)}{\tau_v^{t}(p)-\tau_{u,v}^{t}(p,p)}\right)\Bigg)\Bigg]_*,\nonumber\\\label{eq:updatenb1}
%\frac1{2^{d+7} (d^2+6d+2) \sqrt{t+1}}
\end{eqnarray}
where $c_{p,q}^t:=1-\sum_{x_v\neq p,q} \tau_v^t(x_v)$ and $\tau_{u,v}^{t}(p,p)>0$ is always the unique solution satisfying
\begin{eqnarray}
e^{\Psi^{(u,v)}_{p,q}}\cdot\frac{\tau_u^{t}(p)-\tau_{u,v}^{t}(p,p)}{c_{p,q}^t-\tau_u^{t}(p)-\tau_v^{t}(p)+\tau_{u,v}^{t}(p,p)}\cdot\frac{\tau_v^{t}(p)-\tau_{u,v}^{t}(p,p)}{
\tau_{u,v}^{t}(p,p)}=1&\mbox{and}&
\tau_{u,v}^{t}(p,p)<\min\{\tau_v^{t}(p),\tau_u^{t}(p)\}.\nonumber\\
\label{eq:updatenb2}\end{eqnarray}
In the above, $\Psi^{(v)}_{p,q},\Psi^{(u,v)}_{p,q}$ can be defined in an analogous way to \eqref{eq:defPsi}
using $p,q$ in places of $1,0$.

Once $\tau_{v}^{t+1}(p)$ and $\tau_{u,v}^{t+1}(p,p)$ are updated as per \eqref{eq:updatenb1} and \eqref{eq:updatenb2},
$\tau_v^{t+1}(q)$, $\tau_{u,v}^{t+1}(p,q)$, $\tau_{u,v}^{t+1}(q,p)$ and $\tau_{u,v}^{t+1}(q,q)$ can be also updated as follows:
\begin{eqnarray*}
	\tau_v^{t+1}(q)&=& c_{p,q}^t-\tau_v^{t+1}(p)\\
\tau_{u,v}^{t+1}(p,q)&=& \tau_{u}^{t+1}(p)-\tau_{u,v}^{t+1}(p,p)\\
\tau_{u,v}^{t+1}(q,p)&=& \tau_{v}^{t+1}(p)-\tau_{u,v}^{t+1}(p,p)\\
\tau_{u,v}^{t+1}(q,q)&=& 1-\tau_{v}^{t+1}(p)-\tau_{u}^{t+1}(p)+\tau_{u,v}^{t+1}(p,p).
\end{eqnarray*}
% to $\left\{\tau_v^t(p):v\in V\right\}$ so that they
%satisfy the constraints $\sum_{x_v}\tau_v(x_v)=1$ and
%$\tau_v(x_v)=\sum_{x_u}\tau_{u,v}(x_u,x_v)$ as well as quadratic
%equations corresponding to \eqref{eq:secondmar} of the binary case.
%From similar arguments in the proof of Lemma \ref{lem:key}, one can
%show that it is possible to choose $\alpha$ small enough so that
%updated marginal probabilities
%$\left\{\tau_v^{t+1}(p),\tau_v^{t+1}(q):v\in V\right\}$ and
%$\left\{\tau_{u,v}^{t+1}(p,q):(u,v)\in E\right\}$ are always valid
%(i.e., positive). And hence, 
The convergence rate analysis of this updating rule for non-binary graphical models is almost identical to what
we did in Section \ref{sec:mainpf}. We omit the further details
in this paper.

\section{Experimental Results}\label{sec:simul}

In this section, we report experimental results comparing Algorithm A in Section \ref{sec:two} and the standard
BP algorithm, where we consider a grid-like graph $G=(V,E)$ of 100 nodes such that
$$V=\{(i,j): 0\leq i,j\leq 9\}\qquad E=\{((i,j),(i^{\prime},j^{\prime}))\in V\times V:(i-i^{\prime})^2 +
(j-j^{\prime})^2 = 1\mod 10\}.$$
For the choice of potential functions, two popular MRFs are studied: (1) hard-core model and (2) Ising model.
We also note that as we mentioned in Section \ref{sec:two}, 
we choose the step-size $\frac1{\sqrt{t+100}}$ (instead of $\frac1{\sqrt{t}}$) for Algorithm A.

\paragraph{Hard-core Model.} 
The hard-core model has its origin as a lattice gas model with hard constraints in statistical physics \cite{GauntFisher}, 
but it has also gained much attention in the fields of communication networks, combinatorics,
probability and theoretical computer science. In this model, the potential functions are defined to be
\begin{eqnarray*}
	\psi_v(1)=\lambda \qquad\psi_v(0)=1\qquad\psi_{u,v}(1,1)=0\qquad
	\psi_{u,v}(0,0)=\psi_{u,v}(0,1)=\psi_{u,v}(1,0)=1,
\end{eqnarray*}
where $\lambda>0$ is called `fugacity' (or `activity'). Since Algorithm A requires $\psi(\cdot,\cdot)> 0$ (see Section \ref{sec:one}),
we consider $\psi_{u,v}(1,1)=0.001$ instead of $\psi_{u,v}(1,1)=0$.

The simulation results are reported in Figure \ref{fig1} for the hard-core model with $\lambda=1,2$. 
%, where 
%$x$-axis and $y$-axis are for the estimation of the marginal probability
%at $(0,0)\in V$ and the number of iterations, respectively.
%The left figure of Figure \ref{fig1} is obtained for $\lambda=1$.  
When $\lambda=1$, both algorithms converge: the standard BP 
algorithm needs at least 20 iterations to converge, while Algorithm A 
converges faster, namely in not more than 9 iterations.
%In this case, both algorithms converge: the standard BP algorithm converges in $\geq 20$ iterations, while Algorithm A converges faster
%in $\leq 9$ iterations.
%As reported in the right figure of Figure \ref{fig1},
When $\lambda=2$,
the standard BP algorithm does not converge, while Algorithm A still converges in $15$ iterations. 

\begin{figure}[ht]
\begin{center}
%\centerline{\psfig{figure=hardcore_uniform_fugacity.pdf,width=5cm,angle=0}}
\includegraphics[width=150mm]{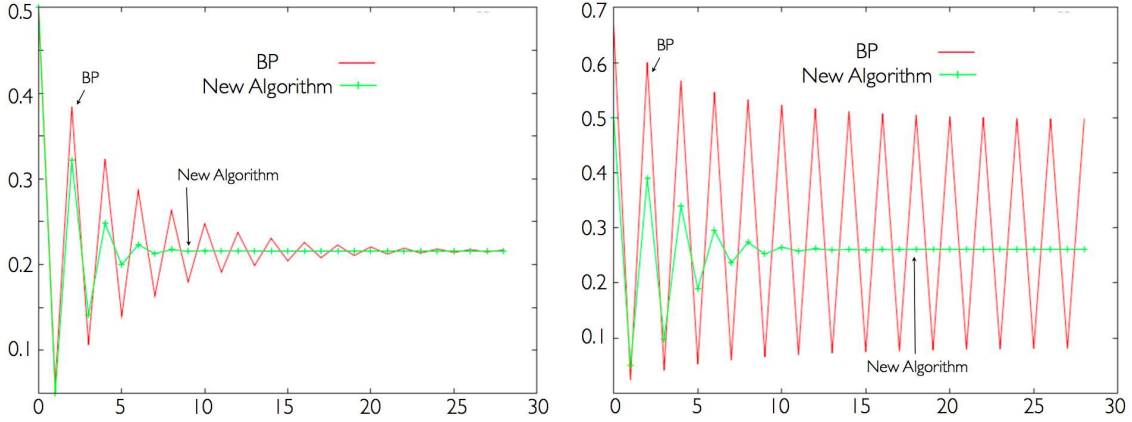}%uniform_fugacity1and2_hardcore.eps}
\caption{Comparisons of the standard BP and our new algorithm, Algorithm A, in hard-core models.
In both figures, $x$-axis and $y$-axis are for 
the number of iterations and for the BP estimation of the marginal probability 
at node $(0,0)$, respectively.
% In both figures, $x$-axis and $y$-axis are for the BP estimation of the marginal probability
%at node $(0,0)$ and the number of iterations, respectively. 
The left and right figures are for $\lambda=1$ and $\lambda=2$, respectively.
In both cases, our algorithm converges faster than BP.} \label{fig1}
\end{center}
\end{figure}

\paragraph{Ising Model.} 
The Ising model, which is named after the physicist 
Ernst Ising \cite{ising}, is a popular mathematical model in statistical mechanics.
%The Ising model is a popular mathematical model in statistical mechanics, where its name is originated from
%the physicist Ernst Ising \cite{ising}.
In this model, the `edge' potential functions are decided as
\begin{eqnarray*}
	\psi_{u,v}(1,1)=\psi_{u,v}(0,0)=e^{\beta J_{u,v}}\qquad \psi_{u,v}(0,1)=\psi_{u,v}(1,0)=1,
\end{eqnarray*}
where $\beta>0$ and $J_{u,v}$ are called `inverse temperature' and `interaction', respectively.
For the `node' potential functions, we choose $\psi_v(1)$ uniformly at random in $[1/2,2]$ independently for each $v\in V$ and
set $\psi_v(0)=1$ deterministically.

We report the simulation results for the Ising model in Figure \ref{fig2}, where
the left and right figures are obtained for $e^{\beta J_{u,v}}=2$ (ferromagnetic)
and $e^{\beta J_{u,v}}=1/2$ (anti-ferromagnetic) for all $(u,v)$ in $E$, respectively. 
%These setups represent `ferromagnetic' interactions ($e^{\beta J_{u,v}}>1$) and
%`anti-ferromagnetic' interactions ($e^{\beta J_{u,v}}<1$), respectively.
In both cases, we observe that the standard BP algorithm converges faster than Algorithm A. 
However, we note that we did not make any significant efforts to choose a better step-size
in Algorithm A so that it converges faster.

\begin{figure}[ht]
\begin{center}
%\centerline{\psfig{figure=hardcore_uniform_fugacity.pdf,width=5cm,angle=0}}
\includegraphics[width=150mm]{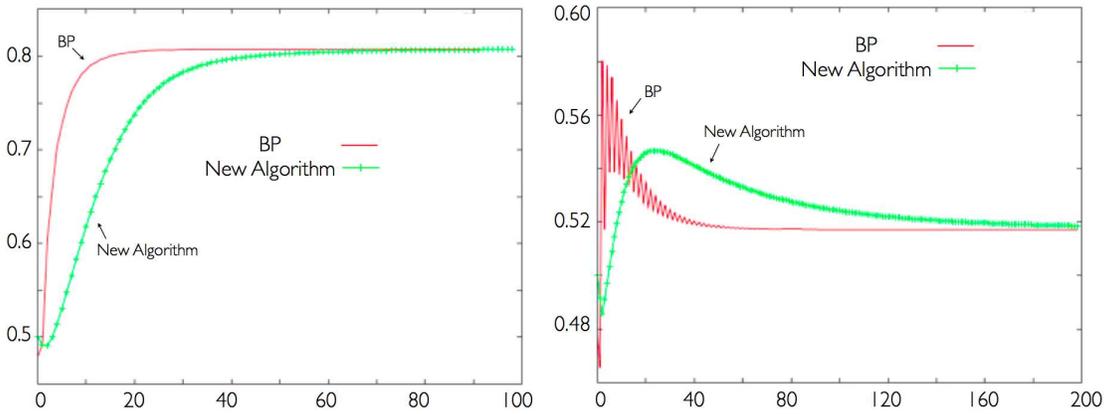}%fugacity[0_5,2]_edgepotentials=2and0_5.001.eps}
\caption{Comparisons of the standard BP and our new algorithm, Algorithm A, in Ising models. 
In both figures, $x$-axis and $y$-axis are for 
the number of iterations and for the BP estimation of the marginal probability 
at node $(0,0)$, respectively.
%In both figures, $x$-axis and $y$-axis are for the BP estimation of the marginal probability
%at node $(0,0)$ and the number of iterations, respectively. 
The left and right figures are for $e^{\beta J_{u,v}}=2$
and $e^{\beta J_{u,v}}=1/2$, respectively. In both cases, our algorithm converges slower than BP.
} \label{fig2}
\end{center}
\end{figure}

\section{Conclusion}

In the last decade, exciting progress has been made on
understanding computationally hard problems in computer science
using a variety of methods from statistical physics. The belief
propagation (BP) algorithm or its variants are among them and
suggest to solve certain `relaxations' of hard problems. In this
paper, we address the question whether the relaxation is indeed
computationally easy to solve in a strong sense. We believe that our
rigorous complexity analysis of the BP-relaxation is an important
step to guarantee the complexity of BP-based algorithms.

\section{Acknowledgment}
We are grateful to Pascal Vontobel, Devavrat Shah and anonymous reviewers 
for their fruitful comments on this paper. We are also grateful to Max Welling
for pointing out the similarity between our algorithm and that in
\cite{Teh}.

%\bibliographystyle{plain}
%\bibliography{biblio}

%\newpage

\appendix

\section{Proof of Lemma \ref{lem:ytom}}\label{sec:pfytom}
To simplify notation, we use $a=b\cdot e^{\pm\varepsilon}$ to mean
$b\cdot e^{-\varepsilon}\leq a\leq b\cdot e^{\varepsilon}$. Now from
$\left|\frac{\partial F}{\partial y_v}\right|\leq\varepsilon$ and
$\left|\frac{\partial F}{\partial y_{u,v}}\right|\leq\varepsilon$ (i.e., the assumption $\|\nabla F(\by)\|_\infty \leq\varepsilon$
in Lemma \ref{lem:ytom}),
we have
%$$e^{\Psi^{(u)}}\cdot \frac{1-y_v-y_u+y_{u,v}}{y_u-y_{u,v}}\cdot\prod_{w\in\mathcal N(u)\setminus v}
%\left(\frac{1-y_u-y_w+y_{w,u}}{1-y_u}\cdot
%\frac{y_u}{y_u-y_{w,u}}\right)=e^{\pm\varepsilon}.
%$$
$$\frac{\psi_u(1)}{\psi_u(0)}\cdot \frac{\psi_{u,v}(1,0)}{\psi_{u,v}(0,0)}\cdot\frac{1-y_v-y_u+y_{u,v}}{y_u-y_{u,v}}\cdot\prod_{w\in\mathcal N(u)\setminus v}
m_{w\to u}=e^{\pm\varepsilon}.
$$
$$e^{\Psi^{(u,v)}}\cdot
\frac{y_u-y_{u,v}}{1-y_u-y_v+y_{u,v}}\cdot\frac{y_v-y_{u,v}}{
y_{u,v}}=e^{\pm\varepsilon}.$$ Using the above inequalities, the
desired conclusion of Lemma \ref{lem:ytom} follows as
\begin{eqnarray*}
m_{u\to
v}&=&\frac{\psi_{u,v}(0,1)}{\psi_{u,v}(0,0)}\cdot\frac{1-y_v-y_u+y_{u,v}}{1-y_v}\cdot
\frac{y_v}{y_v-y_{u,v}}\\
&=&\frac{\psi_{u,v}(0,1)}{\psi_{u,v}(0,0)}\cdot\frac{1+\frac{y_{u,v}}{y_v-y_{u,v}}
}{1+\frac{y_u-y_{u,v}}{1-y_v-y_u+y_{u,v}}}\\
&=&\frac{\psi_{u,v}(0,1)}{\psi_{u,v}(0,0)}\cdot\frac{1+e^{\pm\varepsilon}\cdot
e^{\Psi^{(u,v)}}\cdot \frac{y_u-y_{u,v}}{1-y_u-y_v+y_{u,v}}
}{1+\frac{y_u-y_{u,v}}{1-y_v-y_u+y_{u,v}}}\\
&=&\frac{\psi_{u,v}(0,1)}{\psi_{u,v}(0,0)}\cdot\frac{1+e^{\pm2\varepsilon}\cdot
e^{\Psi^{(u,v)}}\cdot
\frac{\psi_u(1)}{\psi_u(0)}\cdot\frac{\psi_{u,v}(1,0)}{\psi_{u,v}(0,0)}\cdot\prod_{w\in\mathcal
N(u)\setminus v} m_{w\to u}
}{1+e^{\pm\varepsilon}\cdot\frac{\psi_u(1)}{\psi_u(0)}\cdot\frac{\psi_{u,v}(1,0)}{\psi_{u,v}(0,0)}\cdot\prod_{w\in\mathcal
N(u)\setminus v} m_{w\to u}}\\
&=&e^{\pm3\varepsilon}\cdot\frac{\psi_{u,v}(0,1)}{\psi_{u,v}(0,0)}\cdot\frac{1+
e^{\Psi^{(u,v)}}\cdot
\frac{\psi_u(1)}{\psi_u(0)}\cdot\frac{\psi_{u,v}(1,0)}{\psi_{u,v}(0,0)}\cdot\prod_{w\in\mathcal
N(u)\setminus v} m_{w\to u}
}{1+\frac{\psi_u(1)}{\psi_u(0)}\cdot\frac{\psi_{u,v}(1,0)}{\psi_{u,v}(0,0)}\cdot\prod_{w\in\mathcal
N(u)\setminus v} m_{w\to u}}\\
&=&e^{\pm3\varepsilon}\cdot f_{u\to v}\left(\prod_{w\in\mathcal
N(u)\setminus v} m_{w\to u}\right)\\
&=&(1\pm6\varepsilon)\cdot f_{u\to v}\left(\prod_{w\in\mathcal
N(u)\setminus v} m_{w\to u}\right).\end{eqnarray*}

\section{Intuition for Lemma \ref{lem:ytom}}\label{sec:intytom}
We provide some intuition for \eqref{eq:convert} under assuming that $G$ is a tree graph. % of root $v$.
Since the BP marginal estimates $\tau_v(\cdot),\tau_{u,v}(\cdot)$, or equivalently the Bethe stationary points $y_v, y_{u,v}$,
are exact for tree graphs, we have that
$$\frac{1-y_v-y_u+y_{u,v}}{1-y_v}\cdot
\frac{y_v}{y_v-y_{u,v}}~=~\frac{\Pr(x_u=0,x_v=0)}{\Pr(x_v=0)}\cdot
\frac{\Pr(x_v=1)}{\Pr(x_u=0,x_v=1)}~=~\frac{\Pr(x_u=0~|~x_v=0)}{\Pr(x_u=0~|~x_v=1)}.$$
Hence, \eqref{eq:convert} is equivalent to
$${m_{u\to v}}
~=~\frac{\psi_{u,v}(0,1)}{\psi_{u,v}(0,0)}\cdot\frac{\Pr(x_u=0~|~x_v=0)}{\Pr(x_u=0~|~x_v=1)}
~=~\frac{\psi_{u,v}(0,1)}{\psi_{u,v}(0,0)}\cdot\frac{\Pr_T(x_u=0~|~x_v=0)}{\Pr_T(x_u=0~|~x_v=1)},$$
where
$T$ be the subtree of $G$ cutting all branches attached to $v$ except for that including $u$ and $\Pr_T(\cdot)$ denotes
the probability distribution defined by the natural induced graphical model on $T$ (using same potential functions).
Furthermore, it is easy to check that BP fixed-point messages satisfy $\frac{\psi_v(1)}{\psi_v(0)}\cdot {m_{u\to v}} =  \frac{\Pr_T(x_v=1)}{\Pr_T(x_v=0)}$
since $G,T$ are trees.
Therefore, it follows that
\begin{eqnarray*}
\eqref{eq:convert}~~\Leftrightarrow &&	{m_{u\to v}}
	=\frac{\psi_{u,v}(0,1)}{\psi_{u,v}(0,0)}\cdot\frac{\Pr_T(x_u=0~|~x_v=0)}{\Pr_T(x_u=0~|~x_v=1)}\\
\Leftrightarrow	&&\frac{\psi_v(0)}{\psi_v(1)}\cdot  \frac{\Pr_T(x_v=1)}{\Pr_T(x_v=0)}
=\frac{\psi_{u,v}(0,1)}{\psi_{u,v}(0,0)}\cdot\frac{\Pr_T(x_u=0~|~x_v=0)}{\Pr_T(x_u=0~|~x_v=1)}\\
\Leftrightarrow	&&\frac{\psi_v(0)}{\psi_v(1)}\cdot  \frac{\psi_{u,v}(0,0)}{\psi_{u,v}(0,1)}
=\frac{\Pr_T(x_u=0,x_v=0)}{\Pr_{T}(x_u=0,x_v=1)}.
\end{eqnarray*}
In the above, the last equality can be verified easily using the fact that $v$ is a leave of tree $T$.

\end{document}